%% file: AdvSGM_Main.tex
\documentclass[conference]{IEEEtran}
\IEEEoverridecommandlockouts
\makeatletter
\let\NAT@parse\undefined
\makeatother
\usepackage{hyperref}
\usepackage{cite}
\usepackage{lineno}
\usepackage{amsmath,amssymb,amsfonts}
\usepackage{graphicx}
\usepackage{textcomp}
\usepackage[table,xcdraw]{xcolor}
\usepackage{epstopdf}
\usepackage{soul} % for highlight
\usepackage{amsthm} % load proof
\usepackage{multirow}
\usepackage{threeparttable}
\usepackage{ragged2e}
\usepackage{mathrsfs} % load \mathscr
\usepackage[ruled,linesnumbered]{algorithm2e}
\newtheorem{definition}{Definition}
\newtheorem{theorem}{Theorem}

\newtheorem{remark}{Remark}

\usepackage{enumitem}
\usepackage{ulem}
\def\BibTeX{{\rm B\kern-.05em{\sc i\kern-.025em b}\kern-.08em
    T\kern-.1667em\lower.7ex\hbox{E}\kern-.125emX}}

\begin{document}

\title{AdvSGM: Differentially Private Graph Learning via Adversarial Skip-gram Model}

\author{
    \IEEEauthorblockN{Sen Zhang$^{1}$, Qingqing Ye$^{1}$, Haibo Hu$^{1}$, Jianliang Xu$^2$}
    \IEEEauthorblockA{$^1$ The Hong Kong Polytechnic University, Hong Kong}
    \IEEEauthorblockA{$^2$ Hong Kong Baptist University, Hong Kong}
    \IEEEauthorblockA{\{senzhang, qqing.ye, haibo.hu\}@polyu.edu.hk, xujl@comp.hkbu.edu.hk}
}

\maketitle

\input{abstract}


\input{introduction}
\input{preliminaries}
\input{problemDef}

\input{modelFrame}

\input{privAnalysis}
\input{experiments}
\input{relatedWork}

\input{conclusion}

\section*{Acknowledgments}
This work was supported by the National Natural Science Foundation of China (Grant No: 92270123 and 62372122), and the Research Grants Council, Hong Kong SAR, China (Grant No: 15226221, 15224124, C2004-21GF and C2003-23Y).

\bibliographystyle{IEEEtran}
\bibliography{mybibfile}

\end{document}

%% file: abstract.tex
\begin{abstract}
The skip-gram model (SGM), which employs a neural network to generate node vectors, serves as the basis for numerous popular graph embedding techniques. However, since the training datasets contain sensitive linkage information, the parameters of a released SGM may encode private information and pose significant privacy risks. Differential privacy (DP) is a rigorous standard for protecting individual privacy in data analysis. Nevertheless, when applying differential privacy to skip-gram in graphs, it becomes highly challenging due to the complex link relationships, which potentially result in high sensitivity and necessitate substantial noise injection. To tackle this challenge, we present AdvSGM, a differentially private skip-gram for graphs via adversarial training. Our core idea is to leverage adversarial training to privatize skip-gram while improving its utility. Towards this end, we develop a novel adversarial training module by devising two optimizable noise terms that correspond to the parameters of a skip-gram. By fine-tuning the weights between modules within AdvSGM, we can achieve differentially private gradient updates without additional noise injection. Extensive experimental results on six real-world graph datasets show that AdvSGM preserves high data utility across different downstream tasks. 
\end{abstract}

%% file: introduction.tex
\section{Introduction}\label{sec:Intro}
Graph embedding, which has attracted increasing research attention, represents nodes by low-dimensional vectors while preserving the inherent properties and structures of the graph. In this way, well-studied machine learning algorithms can be easily applied for further mining tasks like clustering, classification, and prediction. Skip-gram models (SGMs) are a popular class of graph embedding models thanks to their simplicity and effectiveness, including DeepWalk~\cite{perozzi2014deepwalk}, LINE~\cite{tang2015line}, and node2vec~\cite{grover2016node2vec}. However, SGMs, which capture not only general data characteristics but also specific details about individual data, are vulnerable to adversarial attacks, particularly user-linkage attacks~\cite{korolova2008link} that exploit the linkage information between nodes to infer whether an individual is present in the training dataset. Therefore, node embeddings need to be sanitized with privacy guarantees before they can be released to the public.

Differential privacy (DP)~\cite{dwork2014algorithmic} is a well studied statistical privacy model recognized for its rigorous mathematical underpinnings.
In this paper, we study the problem of achieving privacy-preserving skip-gram for graphs under differential privacy. The work most related to ours is by Ahuja \emph{et al.}~\cite{ahuja2020differentially}, where private learning from sparse location data is achieved by leveraging SGM in conjunction with differentially private stochastic gradient descent (DPSGD)~\cite{abadi2016deep}, a prevalent approach for differentially private training in deep neural networks. Despite its success, this method cannot be extended to skip-gram for graphs, as DPSGD is inherently more suitable for structured data with well-defined individual gradients and does not adapt effectively to graph data. The primary reason for this limitation is that individual examples in a graph are no longer independently computed for their gradients. As a result, achieving differentially private skip-gram for graphs typically encounters significant noise due to high sensitivity, making it challenging to strike a balance between privacy and utility.

Adversarial training, originally developed to protect against adversarial attacks in computer vision~\cite{goodfellow2014explaining}, has proven effective in enhancing the robustness of deep learning models. Recently, adversarial training is used to enhance the robustness and generalizability of skip-gram~\cite{li2023robust}.
Inspired by this, we propose AdvSGM, a differentially private skip-gram that injects differential noise into the activation function of adversarial training. This approach leverages adversarial training to achieve privacy protection while enhancing the utility of skip-gram. Nevertheless, achieving this is challenging as the injected noise may be eliminated during gradient calculations. To overcome this issue, we design a novel adversarial training module by creating two optimizable noise terms associated with the parameters of the skip-gram that require optimization. Through fine-tuning the weights of the skip-gram and adversarial training modules, we demonstrate that the combined noise effectively contributes to privacy protection throughout the optimization process. Additionally, we provide empirical evidence to support the rationale behind this tuning.
Our main contributions are listed as follows.
\begin{itemize}
[leftmargin=5mm]
  \item We present AdvSGM, a novel differentially private skip-gram for graphs that incorporates adversarial training. To our best knowledge, it is the \textbf{\emph{first}} approach which utilizes adversarial training to privatize skip-gram while improving its utility.
  \item We design a novel adversarial training module by introducing two carefully crafted noise terms in activation functions. By fine-tuning the weights between modules, we can achieve differential privacy without requiring additional noise injection during optimization.
  \item Through formal privacy analysis, we prove that our AdvSGM guarantees node-level differential privacy. Extensive experiments on six real graph datasets demonstrate that our solution significantly surpasses existing state-of-the-art private graph embedding methods in both link prediction and node clustering tasks.
\end{itemize}

The remainder of this paper is organized as follows. Section~\ref{Preliminary} covers the preliminaries. In Section~\ref{sec:ProbDef}, we formulate the problem and introduce a naive solution. Our proposed AdvSGM is detailed in Section~\ref{sec:AdvSGM}, and its privacy and time complexity are discussed in Section~\ref{sec:PrivPRM_PrivCompAnalysis}. Comprehensive experimental results are presented in Section~\ref{sec:experiments}. Section~\ref{sec:Related_work} reviews related works. Finally, a conclusion is drawn in Section~\ref{sec:conclu}.

%% file: preliminaries.tex
\section{Preliminaries}\label{Preliminary}
In this section, we will provide a brief overview of adversarial skip-gram model, differential privacy and DPSGD, highlighting their key concepts and important properties. The notations frequently used in this paper are summarized in Table~\ref{tab:SymbolTable}.

\subsection{Notations}
We define an undirected graph as $\mathcal{G} = (V, E)$, where $V$ is the set of nodes and $E$ is the set of edges. For nodes $v_i, v_j \in V$, the pair $(i, j) \in E$ represents an edge in $\mathcal{G}$. The goal of graph embedding is to learn a function $f: V \rightarrow \mathbf{W}$, where $\mathbf{W} \in \mathbb{R}^{|V| \times r}$ is a matrix with embedding dimension $r \ll |V|$, while preserving the intrinsic properties and structures of the original graph $\mathcal{G}$. We denote $\mathbf{v}_i = f(v_i)$ as the vector embedding of node $v_i$.

\begin{table}[htb!] 
  \centering
  \begin{threeparttable}
      \caption{Frequently used symbols}\label{tab:SymbolTable}
      \begin{tabular}{l|l}
        \hline   %  or \cline{col1-col2}
        Symbol & Description \\
        \hline   %  or \cline{col1-col2}
        {$\epsilon, \delta$} & {Privacy parameters} \\
        {$\mathcal{G}, \overline{\mathcal{G}}$} & {Original graph and neighboring graph} \\
        {$\mathcal{A}$} & {A randomized algorithm} \\
        {$\alpha$} & {Order of Rényi divergence} \\
        {$D_\alpha$} & {Rényi divergence of order $\alpha$} \\
        {$|V|, |E|$}  & {Number of nodes and edges in $\mathcal{G}$} \\
        {$\mathbf{x},\mathbf{y}$} & {Lowercase letters denoting vectors}  \\
        {$\mathbf{x}\cdot\mathbf{y}$} & {Inner product between two vectors}  \\
        {$\mathbf{X},\mathbf{Y}$} & {Bold capital letters denoting matrices} \\
        {$\mathbf{I}$} & {Identity matrix} \\
        {$\mathbf{W}_{in},\mathbf{W}_{out}$} & {Embedding matrices of skip-gram} \\
        {$k$} & {Negative sampling number} \\
        {$r$} & {Dimension of low-dimensional vectors} \\
        {$B$} & {Number of samples} \\
        {$\gamma$} & {Sampling probability} \\
        {$G, D$} & {Generator and discriminator} \\
        {${L}_{sgm}^D$} & {Loss function for structure preservation in $D$} \\
        {${L}_{adv}^D$} & {Loss function for adversarial training in $D$} \\
        {$L^G$} & {Loss function of generator $G$} \\
        {$\lambda$} & {Weight between modules} \\
        {$\mathcal{N}_{D,1},\mathcal{N}_{D,2}$} & {Gaussian noise for $D$} \\
        {$\mathcal{N}_{G,1},\mathcal{N}_{G,2}$} & {Gaussian noise for $G$} \\
        {$\Theta^D, \Theta^G$} & {Parameter sets of $D$ and $G$} \\
        {$n^D, n^G$} & {Number of epochs for $D$ and $G$} \\
        {$f \circ g$} & {Composition of two functions} \\
        \hline
      \end{tabular}
  \end{threeparttable}
\end{table}

\subsection{Adversarial Skip-gram Model}
The skip-gram model is a neural network architecture that learns word embeddings by predicting the surrounding words given a target word. In the context of graph embedding, each node in the network can be thought of as a ``word'', and the surrounding words are defined as the nodes that co-occur with the target node in the network.
Inspired by this setting, DeepWalk~\cite{perozzi2014deepwalk} achieves graph embedding generation by treating the paths traversed by random walks over networks as sentences and using skip-gram to learn latent representations of nodes. With the advent of DeepWalk, several skip-gram based graph embedding generation models have emerged,
including LINE~\cite{tang2015line}, PTE~\cite{tang2015pte}, and node2vec~\cite{grover2016node2vec}. 
Recently, adversarial training is used to improve the utility of skip-gram~\cite{li2023robust}. The enhanced model mainly consists of two components: generator and discriminator.

\subsubsection{Generator}
The generator $G$ incorporates two generators $G_{v_i^\prime}$ and $G_{v_j^\prime}$. The neighbor generator $G_{v_j^\prime}= \phi(\mathcal{N}_{G,1}(\sigma^2\mathbf{I})\cdot{\theta_{v_j^\prime}^G})$ for the real node $v_i$ and the neighbor generator $G_{v_i^\prime}=\phi(\mathcal{N}_{G,2}(\sigma^2\mathbf{I})\cdot{\theta_{v_i^\prime}^G})$ for the real node $v_j$, where $\phi$ denotes an activation function that enables non-linear mappings, such as the \emph{Sigmoid} function, $\theta_{v_j^\prime}^G$ and $\theta_{v_i^\prime}^G$ denote the parameters to be optimized, and Gaussian distributions $\mathcal{N}_{G,1}$ and $\mathcal{N}_{G,2}$ generate noise vectors.
Both generators produce embeddings $\mathbf{v}_i^\prime$ and $\mathbf{v}_j^\prime$, representing the generated \textbf{\emph{fake neighbors}} $v_i^\prime$ and $v_j^\prime$ for $v_j$ and $v_i$ respectively. By doing so, we obtain two node pairs: $(v_i,v_j^\prime)$ and $(v_i^\prime,v_j)$, where $v_i$ and $v_j$ are real nodes obtained from the discriminator. The primary purpose of generator $G$ is to deceive the discriminator, and thus its loss function ${L}^G$ is determined as follows:
\begin{equation}\label{eq:NaiveGenLoss}
  \begin{split}
      {L}^G
      =& \min\mathbb{E}_{(i,j) \in E} (\log(1 - F(\mathbf{v}_i\cdot\mathbf{v}_j^\prime)) \\
       & + \log(1 - F(\mathbf{v}_i^\prime\cdot\mathbf{v}_j))).
  \end{split}
\end{equation}

The function $F(\cdot)$ acts as the \emph{discriminant function} within the discriminator. Its output ranges between 0 and 1, representing the probability that the input node pairs $(v_i,v_j^\prime)$ and $(v_i^\prime,v_j)$ are considered positive.

\subsubsection{Discriminator} 
The discriminator is divided into two modules: one module is a graph structure reservation module (i.e., skip-gram) for learning graph structure and the other is an adversarial training module to improve the performance of the model.

\underline{\emph{Skip-gram Module.}}
The purpose of the graph structure preservation module is to preserve the original graph structure in the low-dimensional embedding space. Many skip-gram based graph embedding methods can be used directly as the graph structure preservation module, such as
DeepWalk~\cite{perozzi2014deepwalk}, LINE~\cite{tang2015line},
etc. Taking LINE as an example, for each node pair $(v_i,v_j)$, the loss function $L_{sgm}^D$ is
  \begin{equation}\label{eq:NoveSGM_Eq}
  % \begin{split}
      L_{sgm}^D
    = \log \sigma\left(\mathbf{v}_i \cdot \mathbf{v}_j\right) + \sum_{n=1}^k\mathbb{E}_{v_j^n \sim \mathbb{P}_{n}(v)}\left[\log \sigma\left(-\mathbf{v}_j^n \cdot \mathbf{v}_i\right)\right],
  % \end{split}
  \end{equation}
where $k$ is the number of negative samples, and $\mathbb{P}_n(v)$ is the sampling distribution of \textbf{\emph{negative samples}}.
\begin{remark}
In this paper, positive samples are node pairs drawn from the edge set $E$. For negative sampling, we use the starting node of a positive sample and pair it with a randomly selected node from the node set $V$. As a result, negative samples can include pairs where the edge $(i,j)$ is either in $E$ or not in $E$ (see Algorithm~\ref{Alg:gene_subGra} for details).
\end{remark}

\underline{\emph{Adversarial Training Module.}}
    The objective of the adversarial training module is to determine the authenticity of the input node pair. Given an input node pair $(v_i,v_j)$, we utilize the \emph{Sigmoid} function as the discriminant function, which provides the probability that the node pair is true. In order to achieve this, the generator $G$ generates synthetic neighbor nodes $v_j^\prime$ and $v_i^\prime$ for nodes $v_i$ and $v_j$ respectively. Subsequently, two node pairs $(v_i,v_j^\prime)$ and $(v_i^\prime,v_j)$ are formed. By incorporating the adversarial training module using the discriminant function $F(\cdot)$, the loss function can be derived as
    \begin{equation}\label{eq:ATM_Eq}
      \begin{split}
         {L}_{adv}^D
        =&\mathbb{E}_{(i,j) \in E} (-\log(1 - F(\mathbf{v}_i\cdot\mathbf{v}_j^\prime)) \\ 
        &- \log(1 - F(\mathbf{v}_i^\prime\cdot\mathbf{v}_j))). % 修改了标点
      \end{split}
    \end{equation}

By combining the graph structure preservation module and the adversarial training module, the loss function of the discriminator $D$ can be expressed as
\begin{equation}\label{eq:dis_obj}
    \begin{split}
       L^D = L_{sgm}^D + \lambda L_{adv}^D,
    \end{split}
\end{equation}
where $\lambda>0$ denotes a weight that controls the relative importance of $L_{sgm}^D$ and $L_{adv}^D$.

\begin{remark}
In Eq.\,(\ref{eq:NoveSGM_Eq}), $\sigma(\cdot)$ denotes the Sigmoid function, which is widely used in skip-gram models. In this paper, both the discriminant function $F(\cdot)$ in Eqs.\,(\ref{eq:NaiveGenLoss}) and (\ref{eq:ATM_Eq}) and the activation function $\phi(\cdot)$ also use the Sigmoid function. These settings are beneficial for achieving our goal of utilizing adversarial training to privatize the skip-gram model while improving its utility.
\end{remark}

\subsection{Differential Privacy}\label{sec:DP_properties}
\textbf{$(\epsilon,\delta)$-DP.} Differential privacy (DP)~\cite{dwork2006calibrating} is the prevailing concept of privacy for algorithms operating on statistical databases. In simple terms, DP restricts the extent to which the output distribution of a mechanism can change when there is a small change in its input. When dealing with graph data, the notion of neighboring databases is established by considering two graph datasets: $\mathcal{G}$ and $\overline{\mathcal{G}}$. These datasets are considered neighbors if their difference is limited to at most one edge or node.

\begin{definition}[Edge (Node)-Level DP~\cite{hay2009accurate}]\label{DP_def}
A graph analysis mechanism $\mathcal{A}$ is said to satisfy edge- or node-level $(\epsilon, \delta)$-DP, if, for any two neighboring graphs $\mathcal{G}$ and $\overline{\mathcal{G}}$ (which differ by at most one edge or node), and for all possible sets $O \subseteq Range(\mathcal{A})$, we have
$\mathbb{P}[\mathcal{A}(\mathcal{G}) \in O] \leq \exp(\epsilon) \cdot \mathbb{P}[\mathcal{A}(\overline{\mathcal{G}}) \in O] + \delta$.
\end{definition}

The concept of neighboring datasets $\mathcal{G}$ and $\overline{\mathcal{G}}$ can be categorized into two types. Specifically, if $\overline{\mathcal{G}}$ is derived by replacing a single data point in $\mathcal{G}$, it is referred to as bounded DP~\cite{dwork2006calibrating}. On the other hand, if $\overline{\mathcal{G}}$ is obtained by adding or removing a data point from $\mathcal{G}$, it is known as unbounded DP~\cite{dwork2006differential}. The privacy parameter $\epsilon$ represents the privacy budget, which quantifies the trade-off between privacy and utility in the algorithm. A smaller value of $\epsilon$ indicates a stronger privacy guarantee. The parameter $\delta$ is typically chosen to be very small and is informally referred to as the failure probability, representing the likelihood of a privacy violation.

\emph{Important properties.} DP has the following important properties that help us design complex algorithms from simpler ones:

\begin{theorem}[Sequential Composition~\cite{dwork2014algorithmic}]\label{Theo:RDP_comp}
Let $f(\mathcal{G})$ be $\left(\epsilon_1, \delta_1\right)$-DP and $g(\mathcal{G})$ be $\left(\epsilon_2, \delta_2\right)$-DP, then the mechanism $F(\mathcal{G}) = (f(\mathcal{G}), g(\mathcal{G}))$ which releases both results  satisfies $\left(\epsilon_1+\epsilon_2, \delta_1+\delta_2\right)$-DP.
\end{theorem}

A nice property of DP is that the privacy guarantee is not affected by post-processing steps. 

\begin{theorem}[Post-Processing~\cite{dwork2014algorithmic}]\label{PostProc_Theo}
If $f(\mathcal{G})$ satisfies $(\epsilon,\delta)$-DP, then 
for any function $g$ that does not have direct or indirect access to the private database $\mathcal{G}$, $g(f(\mathcal{G}))$ satisfies $(\epsilon,\delta)$-DP.
\end{theorem} 

\textbf{$(\alpha,\epsilon)$-RDP.} 
In this paper, we adopt an alternative definition of DP known as Rényi Differential Privacy (RDP)~\cite{mironov2017renyi}, which enables stronger results for sequential composition.

\begin{definition}[RDP~\cite{mironov2017renyi}]
Given $\alpha>1$ and $\epsilon>0$, a randomized algorithm $\mathcal{A}$ satisfies $(\alpha, \epsilon)$-RDP if for every adjacent datasets $\mathcal{G}$ and $\overline{\mathcal{G}}$, we have 
$D_\alpha\left(\mathcal{A}(\mathcal{G}) \| \mathcal{A}\left(\overline{\mathcal{G}}\right)\right) \leq \epsilon$,
where $D_\alpha(P \| Q)$ is the Rényi divergence of order $\alpha$ between probability distributions $P$ and $Q$ defined as 
$D_\alpha(P \| Q)=\frac{1}{\alpha-1} \log \mathbb{E}_{x \sim Q}\left[\frac{P(x)}{Q(x)}\right]^\alpha$.
\end{definition}

A key property of RDP is that it can be converted to standard $(\epsilon, \delta)$-DP using Proposition 3 from \cite{mironov2017renyi}, as outlined below.

\begin{theorem}[From RDP to ($\epsilon, \delta)$-DP~\cite{mironov2017renyi}]\label{Theo:RDP_to_DP}
If $\mathcal{A}$ is an $(\alpha, \epsilon)$-RDP algorithm, then it also satisfies $(\epsilon+\log (1 / \delta) / \alpha-1, \delta)$-DP for any $\delta \in(0,1)$.
\end{theorem}

\emph{Gaussian mechanism.} 
Consider a function $f$ that maps a graph $\mathcal{G}$ to an $r$-dimensional output in $\mathbb{R}^r$. To ensure differential privacy for the function $f$, it is common to add random noise to its output. The amount of noise added is determined by the sensitivity of $f$, which is defined as $\Delta_f=\max_{\mathcal{G},\overline{\mathcal{G}}}\|f(\mathcal{G})-f(\overline{\mathcal{G}})\|_2$, where $\mathcal{G}$ and $\overline{\mathcal{G}}$ are neighboring graphs.
A widely used method to achieve RDP is the Gaussian mechanism, which adds Gaussian noise to the output of an algorithm to protect privacy. By adding Gaussian noise with variance $\sigma^2$ to the function $f$, the mechanism can be defined as $\mathcal{A}(\mathcal{G})=f(\mathcal{G})+\mathcal{N}(\sigma^2\mathbf{I})$, where $\mathcal{N}(\sigma^2\mathbf{I})$ represents noise drawn from a Gaussian distribution with variance $\sigma^2$. This results in an $(\alpha,\epsilon)$-RDP algorithm for all $\alpha > 1$, where the privacy parameter $\epsilon$ is given by $\epsilon=\frac{\alpha \Delta_f^2}{2\sigma^2}$.

It is important to note that the concept of sensitivity makes satisfying node-level differential privacy challenging, as changing a single node could potentially remove $|V|-1$ edges in the worst case, where $|V|$ represents the number of nodes. Consequently, a large amount of noise must be added to ensure privacy protection.

\textbf{Amplification by Subsampling.} 
Subsampling introduces a non-zero probability of an added or modified sample not to be processed by the randomized algorithm. Random sampling will enhance privacy protection and reduce privacy loss~\cite{wang2019subsampled,zhu2019poission,mironov2019r}. In this paper, we focus on the ``subsampling without replacement" setup, which adheres to the following privacy amplification theorem for $(\epsilon,\delta)$-DP.

\begin{theorem}[RDP for Subsampled Mechanisms~\cite{wang2019subsampled}]\label{theo:subsample}
Given a dataset of $n$ points drawn from a domain $\mathcal{X}$ and a mechanism $\mathcal{A}$ that accepts inputs from $\mathcal{X}^m$ for $m \leq n$, we consider the randomized algorithm $\mathcal{A}$ for subsampling, which is defined as follows: 1) sample $m$ data points without replacement from the dataset, where the sampling parameter is $\gamma = m / n$, and 2) apply $\mathcal{A}$ to the subsampled dataset. For all integers $\alpha \geq 2$, if $\mathcal{A}$ satisfies $(\alpha, \epsilon(\alpha))$-RDP, then the subsampled mechanism $\mathcal{A} \circ \rm{subsample}$ satisfies $\left(\alpha, \epsilon^{\prime}(\alpha)\right)$ RDP in which
\begin{equation*}\label{eq:RDP_subsamp}
\begin{split}
\epsilon^{\prime}(\alpha) \leq \frac{1}{\alpha-1} \log \big(1+\gamma^2\big(\begin{array}{l}
\alpha \\
2
\end{array}\big) \min \big\{4\big(e^{\epsilon(2)}-1\big), \\
e^{\epsilon(2)} \min \big\{2,\big(e^{\epsilon(\infty)}-1\big)^2\big\}\big\} \\
+\sum_{j=3}^\alpha \gamma^j\big(\begin{array}{l}
\alpha \\
j
\end{array}\big) e^{(j-1) \epsilon(j)} \min \big\{2,\big(e^{\epsilon(\infty)}-1\big)^j\big\}\big).
\end{split}
\end{equation*}
\end{theorem}

\subsection{DPSGD}\label{sec:dpsgd}
One common technique for achieving differentially private training is the combination of noisy Stochastic Gradient Descent (SGD) and advanced composition theorems such as Moments Accountant (MA)~\cite{abadi2016deep}. 
This combination, known as DPSGD, has been widely studied in recent years for publishing low-dimensional node vectors, as \emph{the advanced composition theorems can effectively manage the problem of excessive splitting of the privacy budget during optimization}.
In DPSGD, the gradient $\mathbf{g}\left(x_i\right)$ is computed for each example $x_i$ in a batch with size $B$ of random examples. The $\ell_2$ norm of each gradient is then clipped using a threshold $C$ to control the sensitivity of $\mathbf{g}\left(x_i\right)$. The clipped gradients $\mathbf{g}\left(x_i\right)$ are summed and combined with Gaussian noise $\mathcal{N}\big(C^2 \sigma^2\mathbf{I}\big)$ to ensure privacy. The average of the resulting noisy accumulated gradient $\tilde{\mathbf{g}}$ is then used to update the model parameters. The expression for $\tilde{\mathbf{g}}$ is given by:
\begin{equation}\label{dpsgd_eq1}
  \tilde{\mathbf{g}} = \frac{1}{B}\big(\sum_{i=1}^{B} \mathbf{clip}(\mathbf{g}\left(x_i\right), C)+\mathcal{N}\big(C^2 \sigma^2 \mathbf{I}\big)\big),
\end{equation}
where $\mathbf{clip}(\cdot)$ is a clipping function, and specifically, $\mathbf{clip}(\mathbf{g}\left(x_i\right), C)=\mathbf{g}\left(x_i\right)/\max(1,\frac{\|\mathbf{g}\left(x_i\right)\|_2}{C})$. \emph{For the convenience of presentation, we will replace $\mathbf{clip}(\mathbf{g}\left(x_i\right), C)$ with $\mathbf{clip}(\mathbf{g}\left(x_i\right))$ in the following sections.}

%% file: problemDef.tex
\section{Problem Definition and A First-Cut Solution}\label{sec:ProbDef}
\subsection{Problem Definition}\label{sec:problemDef}
The advantages of graph embeddings, including low dimensionality, information richness, and task independence, have led to a growing willingness among data owners to publish them for data exploration and analysis, rather than disclosing the original graph data. This work specifically concentrates on a differentially private graph embedding generation. Instead of releasing a sanitized version of the original embeddings, we release a privacy-preserving SGM that is trained on the original data while maintaining privacy. Once equipped with this privacy-preserving model, the analyst can generate synthetic embeddings for the intended analysis tasks.

\textbf{Threat Model.} In this work, we consider a white-box attack scenario~\cite{he2019model}. In the threat model, we do not predefine the attacker's computational capabilities, because different from other cryptographic constructs that are based on computational hardness of certain problems, DP protects confidentiality by adding noise. This construct does not depend on the attacker's computational capabilities. Regarding prior knowledge, we assume that the attacker has complete knowledge of the model's architecture and parameters, and can access the published model and all training data, except for the target sample to infer. However, the attacker has no access to the training process, which inevitably involves accessing the target sample.

\textbf{Privacy Model.}
As stated in Sections~\ref{sec:DP_properties} and $\ref{sec:dpsgd}$, DP ensures that although attackers can have all information from the training dataset except one data sample, they still cannot get this data sample after launching attack. 
The post-processing property (Theorem~\ref{PostProc_Theo}) allows us to move the burden of differential privacy to the discriminator, with the generator's differential privacy being guaranteed by the theorem.
Formally, we define differentially private adversarial skip-gram as follows.

\begin{definition}[Adversarial Skip-gram under Bounded DP\footnote{Since our goal is to generate privacy-preserving node embeddings where the number of embeddings matches the number of nodes in the original graph, we define the node-level privacy-preserving graph embedding under bounded DP.}]\label{def:prob_privdef}
Let $\Theta=[\mathbf{W}_{in}, \mathbf{W}_{out}]$ be the set of parameters to be optimized in the skip-gram module (as shown in Eq.\,(\ref{eq:NoveSGM_Eq})), where $\mathbf{v}_i\in\mathbf{W}_{in}$ and $\mathbf{v}_j\in\mathbf{W}_{out}$.
The adversarial skip-gram model $L^D$ satisfies node-level $(\epsilon, \delta)$-DP if two neighboring graphs $\mathcal{G}$ and $\overline{\mathcal{G}}$ differ in only one node and its corresponding edges, and for all possible $\Theta_S\subseteq{Range(L^D)}$, we have
\begin{equation*}\label{Theo:probDef_privModel}
  \mathbb{P}(L^D(\mathcal{G}) \in \Theta_S) \leq \exp (\epsilon) \times \mathbb{P}\left(L^D\left(\overline{\mathcal{G}}\right) \in \Theta_S\right) + \delta,
\end{equation*}
where $\Theta_S$ denotes the set comprising all possible values of $\Theta$.
\end{definition}

Leveraging the robustness to post-processing, we can immediately conclude that the $(\epsilon, \delta)$-private graph embedding generation is robust to graph downstream tasks, as formalized in the following theorem:
\begin{theorem}\label{Theo:probDef_PostProcess}
Let $L^D$ be a private graph embedding generation model that satisfies node-level $(\epsilon,\delta)$-DP, and let $f$ be any graph downstream task that takes the privacy-preserving graph embedding matrix (i.e., $\mathbf{W}_{in}$ or $\mathbf{W}_{out}$) as input. Then, the composition $f \circ {L^D}$ also satisfies node-level $(\epsilon,\delta)$-DP.
\end{theorem}

\subsection{DP-ASGM: A First-Cut Solution}\label{appdix:DPSGM}
As stated in Section~\ref{sec:dpsgd}, DPSGD with the advanced composition mechanism can manage the issue of excessive splitting of the privacy budget during optimization.
One straightforward approach, called as DP-ASGM, to achieve differentially private skip-gram with adversarial training is to perturb the sum of clipped gradients for discriminator. Using $\mathbf{v}$ as a general notation representing either $\mathbf{v}_i$ or $\mathbf{v}_j$ in Eq.\,(\ref{eq:dis_obj}), the noisy gradient $\widetilde{\nabla}_{\mathbf{v}}L^D$ is expressed as follows:
\begin{equation}\label{eq:AdvSGMDPSGD_grad}
  \widetilde{\nabla}_{\mathbf{v}}L^D
  = \frac{1}{B}\big(\sum_{(i,j) \in E_B} \mathbf{clip}\big(\frac{\partial{L^D}}{\partial{\mathbf{v}}}\big)+\mathcal{N}\left(B^2C^2 \sigma^2\mathbf{I}\right)\big),
\end{equation}
in which $E_B$ denotes a batch sample set, and the sensitivity of $\sum_{(i,j) \in E_B}\mathbf{clip}\big(\frac{\partial{L^D}}{\partial{\mathbf{v}}}\big)$ may be up to $BC$ under Definition~\ref{def:prob_privdef}. The main reason is that clipping in DPSGD is designed to minimize the amount of noise added to gradients during the backward pass of each data point. While this method naturally suits structured data with well-defined individual gradients, it cannot be seamlessly extended to deep graph learning. In graph learning, individual examples no longer compute their gradients independently because changing a single node in the graph may affect the gradients of all nodes in a batch. Therefore, the sensitivity of the gradient sum in Eq.\,(\ref{eq:AdvSGMDPSGD_grad}) is proportional to the batch size.

\textbf{Limitation.}
However, the approach described above leads to poor utility. This is primarily due to the large sensitivity, which introduces significant noise and hampers the effectiveness of the optimization process.

%% file: modelFrame.tex
\begin{figure*}[htb]
  \centering
  \includegraphics[width = 6.5in]{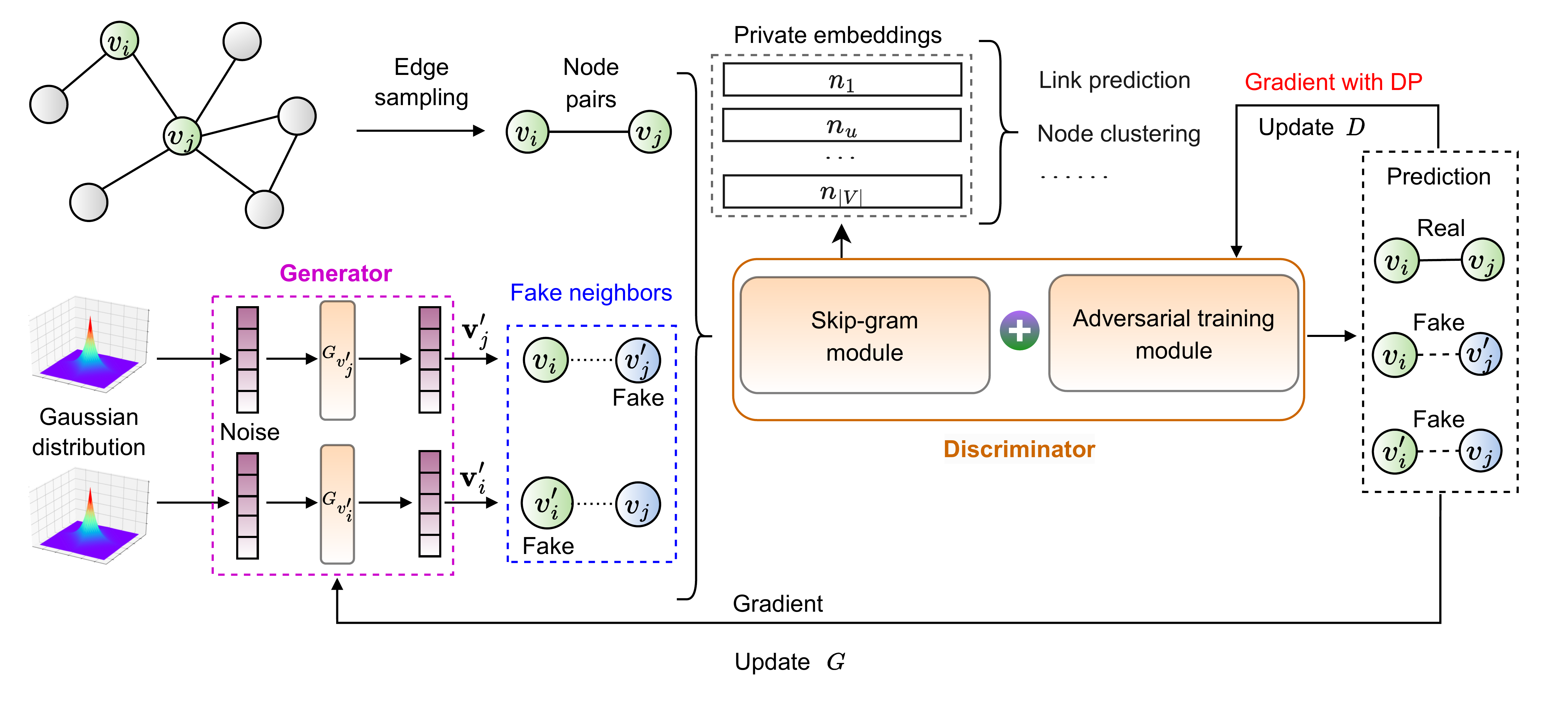}
  \caption{Architecture of AdvSGM. The discriminator can be divided into two modules: skip-gram (graph structure preservation module) for learning the features of the input data, and adversarial training module for improving the performance of skip-gram. Two generators are employed to generate fake neighbors for the real node pair $(v_i,v_j)$. These fake node pairs are designed to deceive the discriminator with a high probability, while the discriminator is trained to distinguish between real and fake node pairs.}
  \label{PrivSG_fig}
\end{figure*}

\section{Our Proposal: AdvSGM}\label{sec:AdvSGM}
To resist the high sensitivity of DP-ASGM, we present AdvSGM, a differentially private skip-gram model that leverages adversarial training to privatize skip-gram while improving its utility. First, we provide an overview of the approach. Next, we describe how we introduce optimizable noise terms and achieve gradient perturbation by adjusting the weights between modules. Finally, we present the complete training algorithm.

\subsection{Overview}
AdvSGM aims to privately learn the embedding vectors for each node in an undirected graph. Fig.\,\ref{PrivSG_fig} illustrates the proposed framework, which consists of two main components: the generator and the discriminator. Given a node pair $(v_i,v_j)$, two generators are employed to generate fake neighbor nodes $v_j^\prime$ and $v_i^\prime$ from two Gaussian distributions. And one discriminator, including skip-gram and adversarial training module, is assigned to distinguish whether the neighborhoods of the generated nodes ($v_j^\prime$ and $v_i^\prime$) are real or fake. Specifically, the adversarial training module is to ensure that discriminator's update satisfies DP, without introducing additional noise.
Following this, the generator will naturally obey DP due to the post-processing property of DP (Theorem~\ref{PostProc_Theo}).
Through the minimax game between the generators and the discriminator, AdvSGM can privately learn more accurate and robust node embeddings.

\textbf{Challenge.}
Based on Eq.\,(\ref{eq:dis_obj}), we take $\frac{\partial{L}^D}{\partial{\mathbf{v}_i}}=\frac{\partial{L}_{sgm}^D}{\partial{\mathbf{v}_i}} + \lambda\frac{\partial{L}_{adv}^D}{\partial{\mathbf{v}_i}}$ as an example to illustrate the challenge in achieving a differentially private discriminator through adversarial training.
Our \emph{target gradient} is to decompose $\frac{\partial{L}^D}{\partial{\mathbf{v}_i}}$ into the form described in Eq.\,(\ref{eq:target_gra}),
\begin{align}
% Target \ Gradient: \
   &\frac{\partial{L}^D}{\partial{\mathbf{v}_i}}
   = \frac{\partial{L}_{sgm}^D}{\partial{\mathbf{v}_i}} + \ell + \mathcal{N}(\sigma^2\mathbf{I}), \label{eq:target_gra}
\end{align}
where we suppose $\lambda\frac{\partial{L}_{adv}^D}{\partial{\mathbf{v}_i}}=\ell + \mathcal{N}(\sigma^2\mathbf{I})$.

To achieve \emph{private gradient} update without additional noise injection, we can clip the first two terms in Eq.\,(\ref{eq:target_gra}) as follows:
\begin{align}
% Private \ Gradient: \
   &\frac{\partial{L}^D}{\partial{\mathbf{v}_i}}
   = \mathbf{clip}(\frac{\partial{L}_{sgm}^D}{\partial{\mathbf{v}_i}} + \ell) + \mathcal{N}(C^2\sigma^2\mathbf{I}) \label{eq:private_gra}.
\end{align}

However, the \emph{real gradient} of $\frac{\partial{L}^D}{\partial{\mathbf{v}_i}}$ is as follows:
\begin{align}
   % \text{Real Gradient:} \ \
   \frac{\partial{L}^D}{\partial{\mathbf{v}_i}}
   =& \frac{\partial{L}_{sgm}^D}{\partial{\mathbf{v}_i}} + \lambda \frac{\partial{L}_{adv}^D}{\partial{\mathbf{v}_i}} \\
   =& \frac{\partial{L}_{sgm}^D}{\partial{\mathbf{v}_i}} + \lambda \frac{\partial{\left(-\log(1 - F(\mathbf{v}_i\cdot\mathbf{v}_j^\prime)\right)}}{\partial{\mathbf{v}_i}} \\ 
   \stackrel{(i)}{=}& \frac{\partial{L}_{sgm}^D}{\partial{\mathbf{v}_i}} +
   \lambda F(\mathbf{v}_i\cdot\mathbf{v}_j^\prime)\cdot\mathbf{v}_j^\prime\label{eq:real_gra},
\end{align}
where the step $(i)$ holds because $F(\cdot)$ is a \emph{Sigmoid} function, and its derivative is $\frac{\partial{F}(x)}{\partial{x}}=F(x)(1-F(x))$.

Clearly, this real gradient cannot be expressed in the form given in Eq.\,(\ref{eq:target_gra}). Therefore, the main challenge is to develop an adversarial loss function that can produce the desired gradient while maintaining high data utility.

\textbf{Our Solution.}
We address this challenge through two steps. 
First, we design a novel adversarial training module by introducing optimizable noise terms, $\mathcal{N}_{D,1}(C^2\sigma^2\mathbf{I})\cdot\mathbf{v}_i$ and $\mathcal{N}_{D,2}(C^2\sigma^2\mathbf{I})\cdot\mathbf{v}_j$, to ensure that
    \begin{equation}\label{eq:target_grad_form}
        \frac{\partial{L}^D}{\partial{\mathbf{v}_i}}
        = (\frac{\partial{L}_{sgm}^D}{\partial{\mathbf{v}_i}} + \ell_1) + \lambda\ell_2\mathcal{N}_{D,1}(C^2\sigma^2\mathbf{I}).
    \end{equation}  
Second, we fine-tune the weights that control relative importance between modules to achieve the desired target gradient form, and more importantly, we show that this configuration is reasonable. 

\subsection{Introducing Optimizable Noise Terms}
To address the second objective while ensuring the first objective is not compromised, we incorporate two optimizable noise terms that correspond to the parameters of a skip-gram, namely $\mathcal{N}_{D,1}(C^2\sigma^2\mathbf{I})\cdot\mathbf{v}_i$ and $\mathcal{N}_{D,2}(C^2\sigma^2\mathbf{I})\cdot\mathbf{v}_j$, into the input of the \emph{Sigmoid} functions. The loss function can be derived as follows:
\begin{align}
  \widetilde{L}_{adv}^D
&=\underbrace{\mathbb{E}_{(i,j) \in E}(-\log(1 - F(\mathbf{v}_i\cdot\mathbf{v}_j^\prime+\mathcal{N}_{D,1}(C^2\sigma^2\mathbf{I})\cdot\mathbf{v}_i)))}_{\widetilde{L}_{adv1}^D} \nonumber\\
&+ \underbrace{\mathbb{E}_{(i,j) \in E}(-\log(1 - F(\mathbf{v}_i^\prime\cdot\mathbf{v}_j+\mathcal{N}_{D,2}(C^2\sigma^2\mathbf{I})\cdot\mathbf{v}_j)))}_{\widetilde{L}_{adv2}^D}. \label{eq:ATM_Eq2}
\end{align}

Based on the Equation, we can get
\begin{align}
        &\frac{\partial{\widetilde{L}_{adv}^D}}{\partial{\mathbf{v}_i}}=\frac{\partial{\widetilde{L}_{adv1}^D}}{\partial{\mathbf{v}_i}} \label{eq:jjjjjj} \\
        =&F(\mathbf{v}_i\cdot\mathbf{v}_j^\prime+\mathcal{N}_{D,1}(C^2\sigma^2\mathbf{I})\cdot\mathbf{v}_i)\cdot
        (\mathbf{v}_j^\prime+\mathcal{N}_{D,1}(C^2\sigma^2\mathbf{I})), \nonumber
\end{align}
and
\begin{align}
   &\frac{\partial{\widetilde{L}_{adv}^D}}{\partial{\mathbf{v}_j}}
  =\frac{\partial{\widetilde{L}_{adv2}^D}}{\partial{\mathbf{v}_j}} \label{eq:hhhhhh} \\
  =&F(\mathbf{v}_i^\prime\cdot\mathbf{v}_j+\mathcal{N}_{D,2}(C^2\sigma^2\mathbf{I})\cdot\mathbf{v}_j) \cdot (\mathbf{v}_i^\prime + \mathcal{N}_{D,2}(C^2\sigma^2\mathbf{I})). \nonumber
\end{align}

By combining $L_{sgm}^D$ and $\widetilde{L}_{adv}^D$, the \emph{novel loss function} of the discriminator $D$ can be expressed as
\begin{equation}\label{eq:dis_func}
    {L_{Nov}^D} = L_{sgm}^D + \lambda_1\widetilde{L}_{adv1}^D + \lambda_2\widetilde{L}_{adv2}^D,
\end{equation}
where $\lambda_1$ and $\lambda_2$ are the weights that control the relative importance between $L_{sgm}^D$ and $\widetilde{L}_{adv}^D$. Using Eqs.\,(\ref{eq:jjjjjj}), (\ref{eq:hhhhhh}), and (\ref{eq:dis_func}), we can derive the target gradient form that aligns with Eq.\,(\ref{eq:target_grad_form}).

According to Eq.\,(\ref{eq:ATM_Eq2}), the novel generator $\widetilde{L}^G$ can be expressed as follows:
\begin{align}
  \widetilde{L}^G 
&= \min\mathbb{E}_{(i, j) \in E}(\log(1 - F(\mathbf{v}_i\cdot\mathbf{v}_j^\prime+\mathcal{N}_{G,1}(C^2\sigma^2\mathbf{I})\cdot\mathbf{v}_i)) \nonumber\\
&\quad + \log(1 - F(\mathbf{v}_i^\prime\cdot\mathbf{v}_j+\mathcal{N}_{G,2}(C^2\sigma^2\mathbf{I})\cdot\mathbf{v}_j))). \label{GenObjFunc}
\end{align}

\subsection{Achieving Gradient Perturbation by Tuning Weights Between Modules}\label{sec:SensAnal}
In this section, we tune the weights $\lambda_1$ and $\lambda_2$ to yield the form like Eq.\,(\ref{eq:private_gra}) in Theorem~\ref{theo:resist_sens}, which implies that we can achieve gradient perturbation without additional noise injection. We further show that these settings are reasonable.

\begin{theorem}\label{theo:resist_sens}
Assuming that \text{$\lambda_1=\frac{1}{F(\mathbf{v}_i\cdot\mathbf{v}_j^\prime+\mathcal{N}_{D,1}(C^2\sigma^2\mathbf{I})\cdot\mathbf{v}_i)}$} and \text{$\lambda_2=\frac{1}{F(\mathbf{v}_i^\prime\cdot\mathbf{v}_j+\mathcal{N}_{D,2}(C^2\sigma^2\mathbf{I})\cdot\mathbf{v}_j)}$}, with the gradient clipping value $C$, we can achieve gradient perturbation without additional noise injection when optimizing the loss function ${L_{Nov}^D}$ in Eq.\,(\ref{eq:dis_func}). In this case, the sensitivity of the gradient sum is $BC$, which implies that node-level privacy is preserved.
\end{theorem}
\begin{proof}
According to Eq.\,(\ref{eq:dis_func}), the gradient of $L_{Nov}^D$ with respect to $\mathbf{v}_i$ is
\begin{align}
&\frac{\partial{L_{Nov}^D}}{\partial\mathbf{v}_i}
=\frac{\partial{L}_{sgm}^D}{\partial{\mathbf{v}_i}} + \lambda_1\frac{\partial\widetilde{L}_{adv1}^D}{\partial{\mathbf{v}_i}} 
=\frac{\partial{L}_{sgm}^D}{\partial{\mathbf{v}_i}} + \label{eq:xxx_grad}\\ 
&\lambda_1 F(\mathbf{v}_i\cdot\mathbf{v}_j^\prime+\mathcal{N}_{D,1}(C^2\sigma^2\mathbf{I})\cdot\mathbf{v}_i)\cdot(\mathcal{N}_{D,1}(C^2\sigma^2\mathbf{I})+\mathbf{v}_j^\prime). \nonumber
\end{align}

Let \text{$\lambda_1=\frac{1}{F(\mathbf{v}_i\cdot\mathbf{v}_j^\prime+\mathcal{N}_{D,1}(C^2\sigma^2\mathbf{I})\cdot\mathbf{v}_i)}$}. By applying gradient clipping, we can rewrite Eq.\,(\ref{eq:xxx_grad}) as
\begin{equation}\label{eq:theo_grad1}
    \frac{\partial{L_{Nov}^D}}{\partial\mathbf{v}_i}
    =\mathbf{clip}(\frac{\partial{L}_{sgm}^D}{\partial{\mathbf{v}_i}} + \mathbf{v}_j^\prime) + \mathcal{N}_{D,1}(C^2\sigma^2\mathbf{I}).
\end{equation}

According to Eq.\,(\ref{eq:dis_func}), the gradient of $L_{Nov}^D$ with respect to $\mathbf{v}_j$ is

\begin{align}
 &\frac{\partial{L_{Nov}^D}}{\partial\mathbf{v}_j}
=\frac{\partial{L}_{sgm}^D}{\partial{\mathbf{v}_j}} + \lambda_2\frac{\partial\widetilde{L}_{adv2}^D}{\partial{\mathbf{v}_j}} 
=\frac{\partial{L}_{sgm}^D}{\partial{\mathbf{v}_j}} + \label{eq:yyy_grad} \\ 
&\lambda_2 F(\mathbf{v}_i^\prime\cdot\mathbf{v}_j+
\mathcal{N}_{D,2}(C^2\sigma^2\mathbf{I})\cdot\mathbf{v}_j)\cdot(\mathcal{N}_{D,2}(C^2\sigma^2\mathbf{I})+\mathbf{v}_i^\prime). \nonumber
\end{align}

Similarly, let \text{$\lambda_2=\frac{1}{F(\mathbf{v}_i^\prime\cdot\mathbf{v}_j+\mathcal{N}_{D,2}(C^2\sigma^2\mathbf{I})\cdot\mathbf{v}_j)}$}. By applying gradient clipping, we can rewrite Eq.\,(\ref{eq:yyy_grad}) as
\begin{equation}\label{eq:theo_grad2}
\frac{\partial{L_{Nov}^D}}{\partial\mathbf{v}_j}
=\mathbf{clip}(\frac{\partial{L}_{sgm}^D}{\partial{\mathbf{v}_j}} + \mathbf{v}_i^\prime) + \mathcal{N}_{D,2}(C^2\sigma^2\mathbf{I}).
\end{equation}

As shown in Eqs.\,(\ref{eq:theo_grad1}) and (\ref{eq:theo_grad2}), we can achieve DP without additional noise.
Using $\widetilde{\nabla}_{\mathbf{v}_i}{L_{Nov}^D}$ and $\widetilde{\nabla}_{\mathbf{v}_j}{L_{Nov}^D}$ to denote the gradients of the batch samples, we have
\begin{equation}\label{eq:rewrit_sens_vi}
\begin{split}
    &\widetilde{\nabla}_{\mathbf{v}_i}{L_{Nov}^D}
  = \frac{1}{B}\sum_{(i,j) \in E_B}\frac{\partial{L_{Nov}^D}}{\partial\mathbf{v}_i} \\
  = &\frac{1}{B}(\sum_{(i,j) \in E_B}\mathbf{clip}(\frac{\partial{L}_{sgm}^D}{\partial{\mathbf{v}_i}} + \mathbf{v}_j^\prime) + \mathcal{N}_{D,1}(B^2C^2\sigma^2\mathbf{I})),
\end{split}
\end{equation}
and
\begin{equation}\label{eq:rewrit_sens_vj}
\begin{split}
      &\widetilde{\nabla}_{\mathbf{v}_j}{L_{Nov}^D}
    = \frac{1}{B}\sum_{(i,j) \in E_B}\frac{\partial{L_{Nov}^D}}{\partial\mathbf{v}_j} \\
    = &\frac{1}{B}(\sum_{(i,j) \in E_B}\mathbf{clip}(\frac{\partial{L}_{sgm}^D}{\partial{\mathbf{v}_j}} + \mathbf{v}_i^\prime) + \mathcal{N}_{D,2}(B^2C^2\sigma^2\mathbf{I})),
\end{split}
\end{equation}
where $E_B$ denotes a batch sample set, and the sensitivity of both $\sum_{(i,j) \in E_B}\mathbf{clip}(\frac{\partial{L}_{sgm}^D}{\partial{\mathbf{v}_i}} + \mathbf{v}_j^\prime)$ in Eq.\,(\ref{eq:rewrit_sens_vi}) and $\sum_{(i,j) \in E_B}\mathbf{clip}(\frac{\partial{L}_{sgm}^D}{\partial{\mathbf{v}_j}} + \mathbf{v}_i^\prime)$ in Eq.\,(\ref{eq:rewrit_sens_vj}) is $BC$. The $BC$ is the upper bound of the gradient sum, indicating that node-level privacy is maintained.
\end{proof}

\begin{remark}
It is worth noting that the sensitivity in Theorem~\ref{theo:resist_sens} is not reduced compared to the naive DP-ASGM in Section~\ref{appdix:DPSGM}. However, in this paper our target is not to decrease sensitivity but rather to achieve node-level DP and enhance utility through adversarial training.
\end{remark}

In what follows, we will constrain the \textit{Sigmoid} function to improve the adaptation of Theorem~\ref{theo:resist_sens}. We also show why the choice of $\lambda_1$ and $\lambda_2$ is appropriate. To simplify the discussion, we use the general notation $\lambda$ to represent either $\lambda_1$ or $\lambda_2$.

\textbf{Constrained \textit{Sigmoid}.}
$\lambda=\frac{1}{F(\cdot)}$ denotes a \emph{matrix}, where each of its elements is greater than 1 since $F(\cdot)$ is a \emph{Sigmoid} function. However, $\lambda$ may be ineffective when some of its elements are exceptionally small.
To tackle this issue, we use a constrained \textit{Sigmoid} function $S(x)=\frac{1}{1+\exp(-x)}$ to replace $F(\cdot)$ and $\sigma(\cdot)$, where the constrained \textit{Sigmoid} function is achieved by constraining $\exp(\cdot)$ with Algorithm~\ref{Alg:Expclip}\footnote{For more details, please refer to \url{https://gist.github.com/patricksurry}.}. This algorithm enables the control of the sharpness of the corners compared with the traditional clipping method. After that, $F(\cdot)$ and $\sigma(\cdot)$ are replaced with $S(\mathbf{x}) = \frac{1}{1+[a,b]}\in[\frac{1}{1+b},\frac{1}{1+a}]$, where $b>a>0$.

\begin{algorithm}\small
\caption{Exponential Clipping}\label{Alg:Expclip}
\KwIn{
Input $x$,
lower bound $a$,
upper bound $b$.}
\KwOut{Clipped exponential value \text{exp\_val}.}
\text{c\_tanh} = $2 / (\exp(2) + 1)$\;
$c = 1 / (2 \cdot \text{c\_tanh})$\;
\If{a is not None and b is not None}
{
    $c /= (b - a) / 2$\;
}
\text{exp\_val} = $\max(\min(x, b), a)$\;
\If{a is not None}
{
    \text{exp\_val} += $\exp(-c \cdot |x-a|)/(2c)$\;
}
\If{b is not None}
{
    \text{exp\_val} -= $\exp(-c \cdot |x-b|)/(2c)$\;
}
\Return \text{exp\_val}\;
\end{algorithm}

Based on the constrained \emph{Sigmoid}, Eq.\,(\ref{eq:NoveSGM_Eq}), and Eq.\,(\ref{eq:ATM_Eq2}), we can rewrite Eq.\,(\ref{eq:dis_func}) as follows:
\begin{equation}
{L_{Nov}^D} = \widehat{L}_{sgm}^D + \lambda_1\widehat{L}_{adv1}^D + \lambda_2\widehat{L}_{adv2}^D,\label{eq:novel_dis_func}    
\end{equation}
where
\begin{equation*}
 \begin{split}
       & \widehat{L}_{sgm}^D
    = \log S\left(\mathbf{v}_i \cdot \mathbf{v}_j\right) + \sum_{n=1}^k\mathbb{E}\left[\log S\left(-\mathbf{v}_j^n \cdot \mathbf{v}_i\right)\right], \\
    &\widehat{L}_{adv1}^D = \mathbb{E}_{(i,j) \in E}(-\log(1 - S(\mathbf{v}_i\cdot\mathbf{v}_j^\prime + \mathcal{N}_{D,1}(C^2\sigma^2\mathbf{I})\cdot\mathbf{v}_i))), \\
    &\widehat{L}_{adv2}^D = \mathbb{E}_{(i,j) \in E}(-\log(1 - S(\mathbf{v}_i^\prime\cdot\mathbf{v}_j + \mathcal{N}_{D,2}(C^2\sigma^2\mathbf{I})\cdot\mathbf{v}_j))),  
 \end{split}   
\end{equation*}
and according to Theorem~\ref{theo:resist_sens}, we have:
\begin{equation*}
\begin{split}
\lambda_1=\frac{1}{S(\mathbf{v}_i\cdot\mathbf{v}_j^\prime+\mathcal{N}_{D,1}(C^2\sigma^2\mathbf{I})\cdot\mathbf{v}_i)}, \\
\lambda_2=\frac{1}{S(\mathbf{v}_i^\prime\cdot\mathbf{v}_j+\mathcal{N}_{D,2}(C^2\sigma^2\mathbf{I})\cdot\mathbf{v}_j)}.
\end{split}
\end{equation*}

\textbf{Rationality of Weight Setting.}
As shown in the previous section, $\lambda_1$ and $\lambda_2$ are presented in matrix form. To evaluate the rationality of these settings, we use $\left|L_{Nov}^D\right|$ as a metric based on Eq.\,(\ref{eq:novel_dis_func}) and set $\lambda=1$ and $\lambda=0.5$ as two baselines, since $\lambda \in (0,1]$ is commonly used in deep learning models. With $a=10^{-5}$ and $b=120$, Fig.\,\ref{Fig:set_of_lambda} displays the loss function values for $\lambda=\frac{1}{S(\cdot)}$ across four datasets: PPI, Facebook, Wiki, and Blog (see dataset details in Section~\ref{sec:exp_set}). The y-axis represents the average $\left|L_{Nov}^D\right|$ value from five independent runs. From this figure, it can be observed that the gap between $\lambda=0.5$ and $\lambda=\frac{1}{S(\cdot)}$ is less than 6 across all tested datasets, while the gap between \text{$\lambda=1$} and \text{$\lambda=\frac{1}{S(\cdot)}$} is less than 2. These small gaps indicate the rationality of the design.

\begin{figure}[htb]
    \centering
    \includegraphics[width=2.5in]{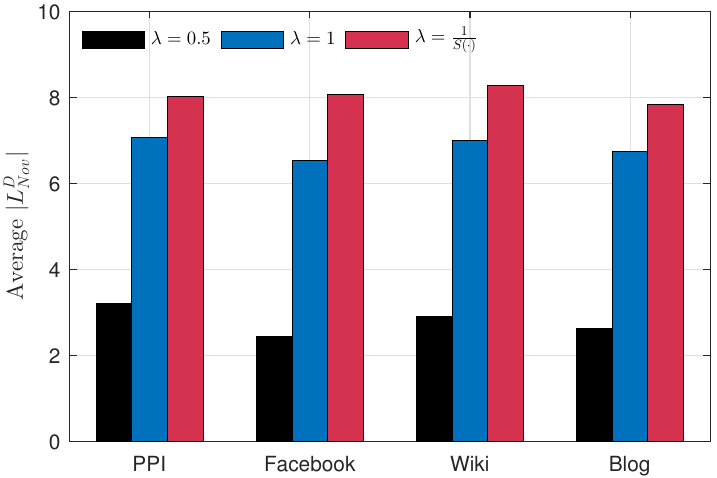}
    \caption{Effect of weight settings across different datasets.}
    \label{Fig:set_of_lambda}
\end{figure}

\subsection{Training Algorithm}
The complete training algorithm for AdvSGM is summarized in Algorithm~\ref{Alg:PrivCNS_Alg}.
In each training epoch, we perform alternating training between the discriminator $D$ and generator $G$.
Specifically, during discriminator training, we first fix $\Theta^G$ and sample $B(k+1)$ fake neighborhoods.
We generate $B$ positive samples and $Bk$ negative samples by Algorithm~\ref{Alg:gene_subGra}.
To compute the probability of privacy amplification more easily, we use $E_B$ and $E_{Bk}$ to optimize $\Theta^D$ respectively, instead of fusing $B$ real samples $E_B$ and $Bk$ negative samples $E_{Bk}$ to optimize $\Theta^D$. Nevertheless, the utility of AdvSGM will be similar when comparing with the use of $E_B$ and $E_{Bk}$ simultaneously because the one-hot encoded vector is used in skip-gram models. As a result, only a fraction of the node vectors in $\mathbf{W}_{in}$ and $\mathbf{W}_{out}$ are updated. In particular, for the gradient with respect to input weight matrix, that is $\frac{\partial L_{sgm}^D}{\partial\mathbf{W}_{in}}$, it is equivalent to taking the derivative for the hidden layer, that is \text{$\frac{\partial L_{sgm}^D}{\partial\mathbf{W}_{in}}=\frac{\partial L_{sgm}^D}{\partial\mathbf{v}_i}$}.
For the gradient with respect to output weight matrix, that is \text{$\frac{\partial L_{sgm}^D}{\partial\mathbf{W}_{out}}=\frac{\partial L_{sgm}^D}{\partial\mathbf{v}_j^n}$} with negative sampling, only a fraction of the node vectors in $\mathbf{W}_{out}$ are updated.
During generator training, we fix $\Theta^D$ and optimize $\Theta^G$ to generate fake neighborhoods that are close to real ones for each node under the guidance of the discriminator $D$. The discriminator and generator play against each other until AdvSGM converges.

\begin{algorithm}\small
\caption{Sample Generation}\label{Alg:gene_subGra}
\KwIn{Original graph $\mathcal{G}$.}
\KwOut{Positive samples $E_B$, negative samples $E_{Bk}$.}
Generate a batch sample set $E_B$ by sampling $B$ edges uniformly at random from $E$\; \label{code:gene_pos_samp}
Generate a batch node set $\mathcal{V}_{B} = \{V_1,\cdots,V_B\}$ by sampling $Bk$ nodes uniformly at random from $V$, where each $V_b$ includes $k$ nodes\;  \label{code:gene_neg_samp}
Set $E_{Bk}=\{\}$\;
\For{$(i,j) \in E_B$}
{\label{code:NegSam_start}
    \For{$V_b\in{\mathcal{V}_B}$}
    {
        Assign $(v_i,v_n)$ to $E_{Bk}$ for $v_n\in{V_b}$\;
    }
}\label{code:NegSam_end}
\Return $E_B$, $E_{Bk}$\;
\end{algorithm}

\begin{algorithm}\small
\caption{AdvSGM Algorithm}\label{Alg:PrivCNS_Alg}
\KwIn{\justifying{Original graph $\mathcal{G}$,
number of samples $B$,
number of training epochs $n^{epoch}$,
number of discriminator's epochs $n^D$,
number of generator's epochs $n^G$,
learning rates of discriminator and generator $\eta_d$, $\eta_g$,
embedding dimension $r$, negative sampling number $k$,
privacy parameters $\epsilon$, $\delta$, and $\sigma$.}}
\KwOut{Privacy-preserving embedding matrix $\Theta^D=\{\mathbf{W}_{in}, \mathbf{W}_{out}\}$.}
Initialize discriminator parameters $\Theta^D$, generator parameters $\Theta^G$\;
\label{code:AssMatri_Initi}
Normalize the parameters of the skip-gram module\;
\label{code:skip_gram_paraNorm}
\For{$epoch = 0$; $epoch < n^{epoch}$}
{
    \tcp{Train Discriminator}
    \For{$n = 0$; $n < n^D$}
    {    \label{code:D_iter_num}
         \tcp{Fake neighbors}
         Sample $B(k+1)$ fake neighbors ${v_i^\prime}^s$, ${v_j^\prime}^s$\;
         \tcp{Generate samples}
         Generate $B$ positive samples, denoted as $E_B$, and $Bk$ negative samples, denoted as $E_{Bk}$, uniformly at random by Algorithm~\ref{Alg:gene_subGra}\; \label{code:samp_edges}
        \For{$E_B$ and $E_{Bk}$}
        {\label{code:resp_train}
        \tcp{Update discriminator}
         Update $\Theta^D$ according to Eq.\,(\ref{eq:novel_dis_func}) with gradient clipping like Eqs.\,(\ref{eq:theo_grad1}) and (\ref{eq:theo_grad2}), which can achieve gradient perturbation without additional noise injection (see Theorem~\ref{theo:resist_sens})\;
         \label{code:grad_perb}
        \tcp{Update privacy accountant of RDP}
        Calculate RDP values\;
        \label{code:cal_RDP}
        $\hat{\delta}\leftarrow$ get privacy spent given the target $\epsilon$\;
        Stop optimization if $\hat{\delta} \geq \delta$\;
        \label{code:stop_opt}
        }
    }
    \tcp{Train Generator}
    \For{$n = 0$; $n < n^G$}
    {
        \label{code:G_iter_num}
        \tcp{Real neighbors}
         Sample $B(k+1)$ real neighbors $v_i^s$, $v_j^s$, where $(v_i,v_j) \in E$\;
         \tcp{Fake neighbors}
         Generate $B(k+1)$ fake neighbors ${v_i^\prime}^s$ for each node $v_i$\;
         \tcp{Fake neighbors}
         Generate $B(k+1)$ fake neighbors ${v_j^\prime}^s$ for each node $v_j$\;
        \tcp{Update generator}
        Update $\Theta^G$ according to Eq.\,(\ref{GenObjFunc})\;
    }
}
\Return $\mathbf{W}_{in}$, $\mathbf{W}_{out}$\;
\end{algorithm}

%% file: privAnalysis.tex
\section{Privacy and Complexity Analysis}\label{sec:PrivPRM_PrivCompAnalysis}
% In this section, we provide a detailed analysis of the privacy and complexity aspects of AdvSGM.
\textbf{Privacy Analysis.} Following Theorem~\ref{theo:subsample}, we adopt the functional perspective of RDP, where $\epsilon$ is a function of $\alpha$, with $1<\alpha<\infty$, and this function is determined by the private algorithm. For ease of presentation, we replace $\left(\alpha, \epsilon^\prime(\alpha)\right)$ with $\left(\alpha, \epsilon^\gamma(\alpha)\right)$ in the following proof, where $\gamma$ denotes the sampling probability. 
\begin{theorem}\label{PrivTheorem}
Given the number of nodes $|V|$, number of edges $|E|$, number of batch size $B$, and negative sampling number $k$, Algorithm~\ref{Alg:PrivCNS_Alg} satisfies node-level
$\big(\alpha, n^{epoch}n^D\epsilon^{\frac{B}{|E|}}(\alpha)+n^{epoch}n^D\epsilon^{\frac{Bk}{|V|}}(\alpha)\big)$-RDP. 
\end{theorem}

\begin{proof}
For the discriminator, the probability of generating $B$ edges uniformly at random from $E$ (line~\ref{code:gene_pos_samp} in Algorithm~\ref{Alg:gene_subGra}) is $\frac{B}{|E|}$.
The probability of generating $Bk$ nodes $\mathcal{V}_B$ uniformly at random from $V$ (line~\ref{code:gene_neg_samp} in Algorithm~\ref{Alg:gene_subGra}) is $\frac{Bk}{|V|}$. Thus, given the known $E_B$, it is easy to determine that the probability of generating $E_{Bk}$ is also $\frac{Bk}{|V|}$ (lines~\ref{code:NegSam_start}-\ref{code:NegSam_end} in Algorithm~\ref{Alg:gene_subGra}). $E_B$ and $E_{Bk}$ are used to train the discriminator's parameters in sequence (line~\ref{code:resp_train} in Algorithm~\ref{Alg:PrivCNS_Alg}). 
Theorem~\ref{theo:resist_sens} reveals that Line~\ref{code:grad_perb} in Algorithm~\ref{Alg:PrivCNS_Alg} achieves privacy protecting through gradient perturbation.
According to the sequential composition property (Theorem~\ref{Theo:RDP_comp}), after $n^{epoch}n^D$ iterations, the discriminator is node-level $\big(\alpha, n^{epoch}n^D\epsilon^{\frac{B}{|E|}}(\alpha)+n^{epoch}n^D\epsilon^{\frac{Bk}{|V|}}(\alpha)\big)$-RDP.  
For the generator, the post-processing property (Theorem~\ref{PostProc_Theo}) ensures that the generator's privacy level aligns with that of the discriminator. Therefore, Algorithm~\ref{Alg:PrivCNS_Alg} obeys node-level $\big(\alpha, n^{epoch}n^D\epsilon^{\frac{B}{|E|}}(\alpha)+n^{epoch}n^D\epsilon^{\frac{Bk}{|V|}}(\alpha)\big)$-RDP. Finally, Theorem~\ref{Theo:RDP_to_DP} is applied to convert the RDP back to the standard node-level DP.
\end{proof}

\textbf{Complexity Analysis.}
To analyze the time complexity of AdvSGM in Algorithm~\ref{Alg:PrivCNS_Alg}, we can break down the computations involved in each major step. 
The outer loop runs for $n^{epoch}$ epochs, where $n^{epoch}$ denotes the number
of training epochs. Within each epoch, the discriminator training loop, which consists of $n^D$ iterations (see Line~\ref{code:D_iter_num} of Algorithm~\ref{Alg:PrivCNS_Alg}), samples $B(k+1)$ fake neighbors and generates $B(k+1)$ positive and negative edges for updating the discriminator parameters. The time complexity of updating the DP cost using RDP depends on the specific implementation of RDP. Different versions of RDP may result in slight differences in time complexity, but according to~\cite{bu2023differentially}, these implementations all have asymptotic complexity of $\mathcal{O}(Br\gamma)$, where $\gamma$ denotes the sampling probability. Therefore, the time complexity for updating the discriminator within one epoch can be approximated as $\mathcal{O}(n^DB(k+1)r+n^DrB\gamma)$. 
Within each generator training iteration, the generator loop, which consists of $n^G$ iterations (see Line~\ref{code:G_iter_num} of Algorithm~\ref{Alg:PrivCNS_Alg}), samples $B(k+1)$ real neighbors and generates $2B(k+1)$ fake neighbors for updating the generator parameters. The time complexity for updating the generator within one epoch can be approximated as $\mathcal{O}(n^GB(k+1)r)$. Therefore, the overall time complexity for the entire algorithm, running for $n^{epoch}$ epochs, is $\mathcal{O}(n^{epoch}n^DBkr+n^{epoch}n^DBr\gamma+ n^{epoch}n^GBkr)$. The complexity is linear with respect to the iteration number and batch size, so our method is scalable and can be applied to large-scale networks.

%% file: experiments.tex
\section{Experiments}\label{sec:experiments}
In this section, we evaluate the performance of AdvSGM~\footnote{Our code is available at \url{https://github.com/sunnerzs/AdvSGM}.} in two downstream tasks: link prediction and node clustering.
Link prediction is a commonly used benchmark task in graph learning models, which tests an algorithm's ability to predict missing links in a graph, assessing how well it infers potential connections based on existing node information. Node clustering evaluates an algorithm's ability to identify and group nodes into meaningful clusters based on their similarities, assessing how well it discovers inherent group structures. 
We aim to address the following four questions:
\begin{itemize}
[leftmargin=5mm]
  \item How much do the parameters impact the performance of AdvSGM? (see Section~\ref{exp:ParaImp})
  \item How is the performance of different differentially private skip-gram models? (see Section~\ref{sec:eps_on_PrivSGMs})
  \item How much does the privacy budget influence the performance of AdvSGM and other private graph models in link prediction? (see Section~\ref{exp:eps_on_lp})
 \item How much does the privacy budget influence the performance of AdvSGM and other private graph models in node clustering? (see Section~\ref{exp:eps_on_nclu})
\end{itemize}

\subsection{Experimental Setup}\label{sec:exp_set}
\textbf{Datasets.}
To comprehensively evaluate our proposed method, we conduct extensive experiments on the following six real-world graph datasets: namely PPI, Facebook, Wiki, Blog, Epinions, DBLP. Since we focus on simple graphs in this work, all datasets are pre-processed to remove self-loops. The details of the datasets are provided as follows.

\begin{itemize}
[leftmargin=5mm]
\item PPI~\cite{stark2006biogrid}: This dataset represents a human Protein-Protein Interaction network, consisting of 3,890 nodes from 50 classes and 76,584 edges. The nodes represent proteins, and the edges indicate interactions between these proteins.
\item Facebook\footnote{\label{data:facebook}\url{https://snap.stanford.edu/data/ego-Facebook.html}}: This dataset represents a social network with 4,039 nodes and 88,234 edges. The nodes represent users, and the edges represent the relationships between them.
\item Wiki\footnote{\url{https://www.mattmahoney.net/dc/text.html}}: This dataset consists of a network of hyperlinks between Wikipedia pages, with 4,777 nodes from 40 categories and 92,517 edges. Each node represents a Wikipedia page, and each edge represents a hyperlink between two pages.
\item Blog\footnote{\label{data:blog}\url{http://datasets.syr.edu/datasets/BlogCatalog3.html}}: This dataset is an online social network containing 10,312 nodes from 39 categories and 333,983 edges. The nodes represent users, and the edges represent relationships between these users.
\item Epinions\footnote{\url{https://snap.stanford.edu/data/soc-Epinions1.html}}: The Epinions dataset is a trust network with 75,879 nodes and 508,837 edges. The nodes represent users, and the directed edges represent trust relationships between them.
\item DBLP\footnote{\label{data:dblp}\url{https://www.aminer.cn/citation}}: This dataset represents a scholarly network with 2,244,021 nodes and 4,354,534 edges. The nodes represent papers, authors, and venues, and the edges represent authorships and the venues where papers are published.
\end{itemize}

\textbf{Evaluation Metrics.}
For the link prediction task, all existing links in each dataset are randomly split into a training set 90\% and a test set 10\%. For the test set, we sample the same number of node pairs without connected edges as negative test links to evaluate link prediction performance. For the training set, we additionally sample the same number of node pairs without edges to construct negative training data. AUC is used to measure performance. 
For the node clustering task, we feed the embedding vectors generated by each algorithm into a node clustering algorithm. Following~\cite{nguyen2018learning}, we adopt the Affinity Propagation algorithm~\cite{frey2007clustering} as the clustering method and evaluate the clustering results in terms of mutual-information (MI). For each result, we measure it over five experiments to report the average value. A larger AUC or MI implies a better utility.

\textbf{Competitive Methods.}
To establish a baseline for comparison, we utilize four state-of-the-art private graph learning methods, namely DPGGAN~\cite{yang2020secure}, DPGVAE~\cite{yang2020secure}, GAP~\cite{sajadmanesh2023gap}, and DPAR~\cite{zhang2024dpar}.
In this study, we simulate a scenario where the graphs only contain structural information, while GAP and DPAR rely on node features.
To ensure a fair evaluation, similar to prior research~\cite{du2022understanding}, we use randomly generated features as inputs for GAP and DPAR.
Additionally, we design different versions of skip-gram for comparison: SGM (No DP), DP-SGM and DP-ASGM. Here, SGM (No DP) refers to the original skip-gram model (i.e., LINE~\cite{tang2015line}), DP-SGM denotes the skip-gram model with DPSGD, and DP-ASGM represents the skip-gram model with adversarial training based on DPSGD. 
Note that skip-gram (see Eq.\,(\ref{eq:NoveSGM_Eq})) incorporates two vectors for each node, namely a context (output) vector and a node (input) vector. However, during our testing phase, we do not observe any performance improvement by utilizing both vectors together. We only employ the node vectors for our experiments, similar to previous methods~\cite{tang2015line, zhou2017scalable}.

\textbf{Parameter Settings.}
In the link prediction and node clustering tasks, we use training epochs $n^{epoch}=50$ for each dataset. In each epoch, we set $n^D=15$ and $n^G=5$ for both tasks. We set the embedding dimension $r=128$. It is worth noting that we do not specifically highlight the impact of $r$ as it is commonly used in various graph embedding methods~\cite{perozzi2014deepwalk,tang2015line,lai2017prune,tu2018unified}. To maintain compatibility with most skip-gram based methods~\cite{tang2015line,du2022understanding,tu2018unified}, we set \text{$k = 5$}. We normalize the parameters of skip-gram module in AdvSGM to ensure that \text{$C=1$}. For privacy parameters, we follow existing works~\cite{shokri2015privacy, abadi2016deep, yang2020secure} to fix \text{$\delta=10^{-5}$} and \text{$\sigma=5$}. Then, we vary the privacy budget $\epsilon$ among the values $\{1, 2, 3, 4, 5, 6\}$ to see how much utilities are preserved under different privacy budgets. 
Also, we vary the learning rates $\eta_d$ and $\eta_g$, as well as the batch size $B$ to verify the effect on the utility of AdvSGM. We fix $a=10^{-5}$ to ensure that the upper bound of $S(\mathbf{x})$ approaches 1, and vary $b$ to verify the effect on the utility of AdvSGM. To ensure consistency with the original papers, we utilize the official GitHub implementations for DPGGAN, DPGVAE, GAP and DPAR. We replicate the experimental setup as described in those papers.

\subsection{Impact of Parameters}\label{exp:ParaImp}
In this section, taking link prediction as an example, we investigate the effects of different parameters on the performance of AdvSGM. We conduct experiments on all datasets with varying the learning rate $\eta$ from 0.01 to 0.3, the batch size $B$ from 16 to 512, and the parameter $b$ from 40 to 140. 

\subsubsection{Parameter $\eta$} In this experiment, we examine the influence of the learning rate on the performance of AdvSGM in the context of link prediction task. We conduct experiments on three datasets: PPI, Facebook, and Blog, with different values of the learning rates $\eta_d$ and $\eta_g$: 0.01, 0.05, 0.1, 0.15, 0.2, 0.25 and 0.3. The results are summarized in Table~\ref{Tab:impact_of_learnRate}. From the table, we can observe that the best performance is achieved when \text{$\eta_d=\eta_g=0.1$} for the tested three datasets. Therefore, we set \text{$\eta_d=\eta_g=0.1$} as the default parameter configuration for all subsequent experiments. Additionally, we consistently find that all standard deviations are no more than 0.02 across all datasets, indicating that AdvSGM exhibits good stability.

\begin{table}[htb] 
\centering
\caption{Summary of AUC values with different $\eta_d$ and $\eta_g$, given $\epsilon=6$ (Result: average AUC $\pm$ standard deviation). \textbf{Bold}: best}\label{Tab:impact_of_learnRate}
\begin{tabular}{c|ccc}
\hline
$\eta_d=\eta_g$ & PPI                          & Facebook                     & Blog               \\ \hline
0.01	&0.5327$\pm$0.0092 	&0.6289$\pm$0.0024 	&0.5393$\pm$0.0065 \\	     	 	
0.05	&0.5753$\pm$0.0050 	&0.6834$\pm$0.0010 	&0.5623$\pm$0.0081 \\ 	     	 	
0.1	    &\textbf{0.6095$\pm$0.0101} 	&\textbf{0.7070$\pm$0.0064} 	&\textbf{0.6595$\pm$0.0200} \\ 	     	 	
0.15	&0.5687$\pm$0.0096 	&0.6478$\pm$0.0048 	&0.5876$\pm$0.0055 \\      	 	
0.2	    &0.5246$\pm$0.0107 	&0.6057$\pm$0.0021 	&0.5368$\pm$0.0023 \\      	 	
0.25	&0.5264$\pm$0.0044 	&0.5673$\pm$0.0027 	&0.5216$\pm$0.0043 \\      	 	
0.3	    &0.5224$\pm$0.0077 	&0.5471$\pm$0.0029 	&0.5147$\pm$0.0004 \\ \hline
\end{tabular}
\end{table}

\subsubsection{Parameter $B$}
In this experiment, we investigate the impact of the parameter $B$ on the performance of AdvSGM in terms of link prediction task. Specifically, we consider different values of $B$, namely, $16, 32, 64, 128, 256, 512$ for all datasets. As illustrated in Table~\ref{Tab:impact_of_batchSize}, for PPI and Facebook, the optimal results are achieved when \text{$B=128$}.
For Blog, while the AUC performs better when \text{$B=512$}, using \text{$B=128$} and \text{$B=256$} can still yield competitive results in terms of AUC. Therefore, we set \text{$B=128$} as the default parameter configuration for all subsequent experiments. Also, we consistently observe that all standard deviations remain consistently below 0.02 across all datasets, indicating that AdvSGM is very stable.

\begin{table}[htb]
\centering
\caption{Summary of AUC values with different $B$, given $\epsilon=6$ (Result: average AUC $\pm$ standard deviation). \textbf{Bold}: best}\label{Tab:impact_of_batchSize}
\begin{tabular}{c|ccc}
\hline
$B$ & PPI                          & Facebook                     & Blog                         \\ \hline
16	&0.5057$\pm$0.0087 	&0.5273$\pm$0.0042 	&0.5052$\pm$0.0010 \\	 	 	
32	&0.5250$\pm$0.0040 	&0.5602$\pm$0.0067 	&0.5142$\pm$0.0053 \\ 	 	
64	&0.5600$\pm$0.0019 	&0.6403$\pm$0.0024 	&0.5465$\pm$0.0049 \\ 	 	
128	&\textbf{0.6095$\pm$0.0101} 	&\textbf{0.7070$\pm$0.0064} 	&0.6595$\pm$0.0200 \\ 	 	
256	&0.5676$\pm$0.0069 	&0.6599$\pm$0.0031 	&0.6625$\pm$0.0076 \\ 	 	
512	&0.5077$\pm$0.0056 	&0.5286$\pm$0.0053 	&\textbf{0.6722$\pm$0.0154} \\ \hline
\end{tabular}
\end{table}

\subsubsection{Parameter $b$}
Recall from Section~\ref{sec:SensAnal} that both $\lambda_1$ and $\lambda_2$ are defined as $\frac{1}{S(\cdot)}$, where $S(\cdot)$ is constrained within the range $[\frac{1}{1+b}, \frac{1}{1+a}]$. In this experiment, we set $a$ to $10^{-5}$ to ensure that $S(\mathbf{\cdot})$ converges to 1. With $b$ varying from 40 to 140, Table~\ref{Tab:impact_of_b} illustrates the impact of parameter $b$ on AdvSGM for link prediction. As $b$ increases, there is a gradual improvement in performance, and the optimal results across all datasets are achieved when \text{$b=140$}. However, the performance gain for \text{$b=140$} compared to \text{$b=120$} is minimal. Therefore, we choose \text{$b=120$} as the default parameter configuration for all subsequent experiments.

\begin{table}[htb]
\centering
\caption{Summary of AUC values with different $b$, given $\epsilon=6$ (Result: average AUC $\pm$ standard deviation). \textbf{Bold}: best}\label{Tab:impact_of_b}
\begin{tabular}{c|ccc}
\hline
$b$ & PPI                          & Facebook                     & Blog                           \\ \hline
40	&0.5053$\pm$0.0033 	&0.5261$\pm$0.0044 	&0.5087$\pm$0.0025 \\ 	 	 	
60	&0.5203$\pm$0.0054 	&0.5557$\pm$0.0052 	&0.5142$\pm$0.0037 \\  	 	
80	&0.5302$\pm$0.0036 	&0.6022$\pm$0.0019 	&0.5424$\pm$0.0107 \\  	 	
100	&0.5694$\pm$0.0063 	&0.6627$\pm$0.0029 	&0.5905$\pm$0.0129 \\  	 	
120	&0.6095$\pm$0.0101 	&0.7070$\pm$0.0064 	&0.6595$\pm$0.0200 \\  	 	
140	&\textbf{0.6584$\pm$0.0140} 	&\textbf{0.7434$\pm$0.0057} 	&\textbf{0.7157$\pm$0.0175} \\ \hline
\end{tabular}
\end{table}

\begin{table}[htb]
\centering
\caption{Summary of AUC/MI values with different $\epsilon$. \textbf{Bold}: best}\label{Tab:eps_on_PrivSGMs}
\begin{tabular}{c|ccc|cc}
\hline
\multirow{2}{*}{Algorithms} & \multicolumn{3}{c|}{AUC}                                              & \multicolumn{2}{c}{MI}            \\ \cline{2-6}
                           & PPI             & Facebook        & Blog            & PPI             & Blog            \\ \hline
SGM(No DP)                 &0.5924          &0.6447          &0.6214          &0.7385           &0.5363          \\ 
AdvSGM(No DP)   &\textbf{0.6914}   &\textbf{0.7588}   &\textbf{0.7078}          &\textbf{1.1927}   &\textbf{0.9942}  \\ \hline
DP-SGM($\epsilon$=1)        & 0.5063          & 0.5077          & 0.5030          & 0.5768          & 0.3869          \\
DP-ASGM($\epsilon$=1)     & 0.5077          & 0.5071          & 0.5036          & 0.5798          & 0.4530          \\
AdvSGM($\epsilon$=1)        &\textbf{0.5083} &\textbf{0.5398} &\textbf{0.5187} &\textbf{0.5810} &\textbf{0.5284} \\ \hline
DP-SGM($\epsilon$=2)        & 0.5076          & 0.5040          & 0.5027          & 0.5704          & 0.4654          \\
DP-ASGM($\epsilon$=2)     & 0.5078          & 0.5069          & 0.5028          & 0.5820          & 0.4731          \\
AdvSGM($\epsilon$=2)        & \textbf{0.5152} & \textbf{0.5459} & \textbf{0.5964} & \textbf{0.5847} & \textbf{0.6787} \\ \hline
DP-SGM($\epsilon$=3)        & 0.5095          & 0.5033          & 0.5027          & 0.5686          & 0.4510          \\
DP-ASGM($\epsilon$=3)     & 0.5047          & 0.5052          & 0.5054          & 0.6130          & 0.4601          \\
AdvSGM($\epsilon$=3)        & \textbf{0.5271} & \textbf{0.5779} & \textbf{0.6593} & \textbf{0.6332} & \textbf{0.7921} \\ \hline
DP-SGM($\epsilon$=4)        & 0.5057          & 0.5047          & 0.5020          & 0.7200          & 0.4679          \\
DP-ASGM($\epsilon$=4)     & 0.5061          & 0.5066          & 0.5023          & 0.7225          & 0.4560          \\
AdvSGM($\epsilon$=4)        & \textbf{0.5655} & \textbf{0.6368} & \textbf{0.6629} & \textbf{0.7374} & \textbf{0.7926} \\ \hline
DP-SGM($\epsilon$=5)        & 0.5074          & 0.5085          & 0.5022          & 0.5512          & 0.3663          \\
DP-ASGM($\epsilon$=5)     & 0.5096          & 0.5097          & 0.5040          & 0.6189          & 0.4549          \\
AdvSGM($\epsilon$=5)        & \textbf{0.5878} & \textbf{0.6829} & \textbf{0.6688} & \textbf{0.9038} & \textbf{0.7978} \\\hline
DP-SGM($\epsilon$=6)        & 0.5077          & 0.5054          & 0.5020          & 0.5851          & 0.3711          \\
DP-ASGM($\epsilon$=6)     & 0.5084          & 0.5079          & 0.5024          & 0.6140          & 0.3668          \\
AdvSGM($\epsilon$=6)       & \textbf{0.6095} & \textbf{0.7070} & \textbf{0.6595} & \textbf{1.0818} & \textbf{0.8777} \\ \hline
\end{tabular}
\end{table}

\subsection{Comparison Between Private Skip-gram Models}\label{sec:eps_on_PrivSGMs}
In this section, we compare different private skip-gram models. Table~\ref{Tab:eps_on_PrivSGMs} presents the AUC and MI results for all the methods. From this table, we make three important observations. First, AdvSGM(No DP) achieves better utility on all tested datasets compared to SGM(No DP), which indicates that the adversarial training module can effectively improve the performance of SGM(No DP). 
Second, AdvSGM significantly outperforms the other private models, DP-SGM and DP-ASGM, across all privacy budgets and datasets in terms of both AUC and MI. This suggests that our proposed adversarial training module, which introduces two additional noise terms into the activation functions, is more effective than the direct perturbation approach based on DPSGD used in DP-SGM and DP-ASGM. Another important observation from Table~\ref{Tab:eps_on_PrivSGMs} is that as $\epsilon$ increases, AdvSGM tends to produce results comparable to those of the non-private SGM. Notably, at $\epsilon=6$, AdvSGM even surpasses the non-private SGM, indicating that our adversarial training module not only preserves the privacy of the skip-gram but also enhances its utility.
Overall, the findings clearly demonstrate the advantages of AdvSGM over other privacy-preserving skip-gram models, highlighting its ability to achieve a better balance between privacy and utility.

\begin{figure*}[htb]
  \centerline{
  \includegraphics[width=6.5in]{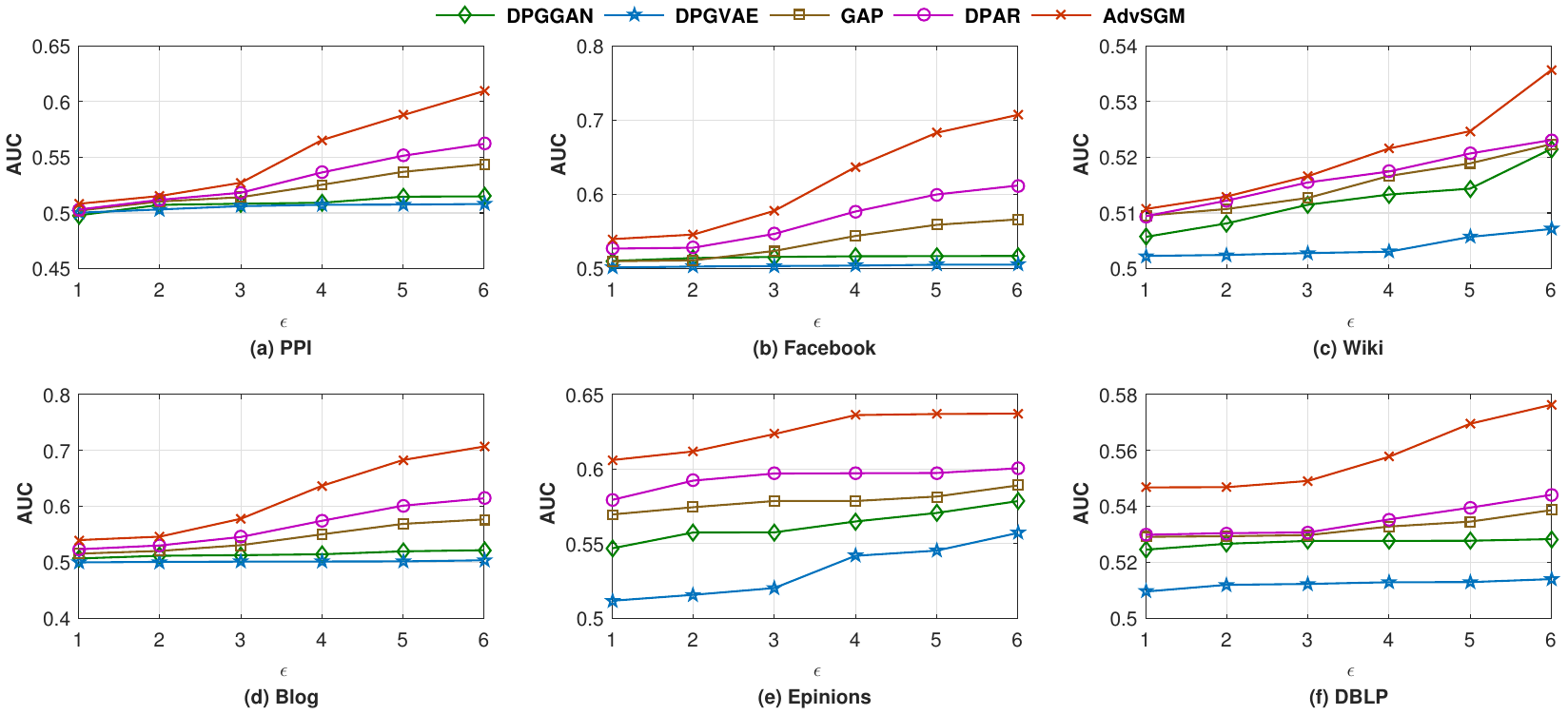}
  }
  \caption{Impact of Privacy Budget on Link Prediction.}
  \label{fig:Exp_PrivBud_on_lp_fig}
\end{figure*}

\begin{figure*}[htb]
  \centering
  \includegraphics[width=6.5in]{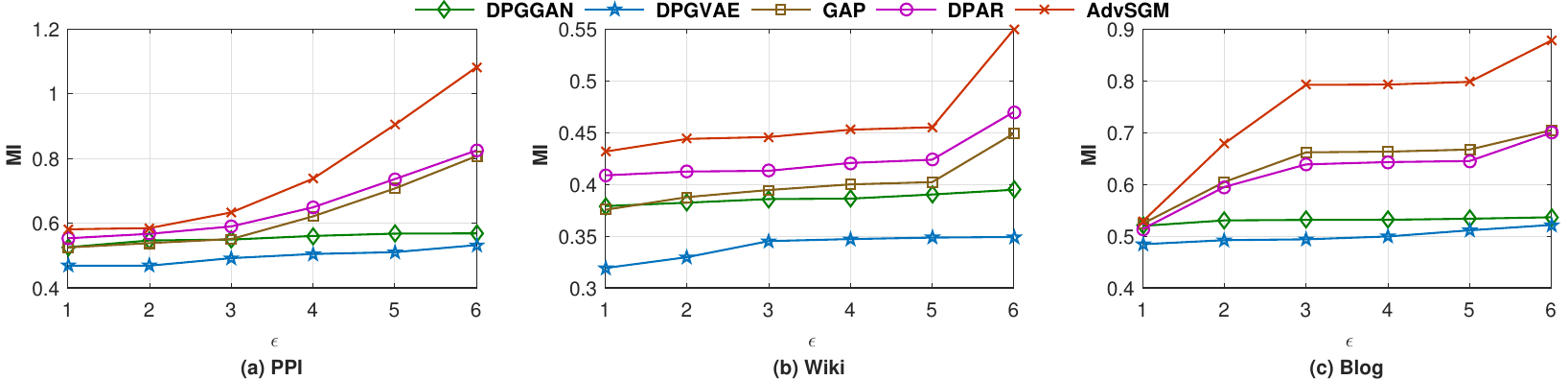}
  \caption{Impact of Privacy Budget on Node Clustering.}\label{fig:Exp_PrivBud_on_NC_fig}
\end{figure*}

\subsection{Impact of Privacy Budget on Link Prediction}\label{exp:eps_on_lp}
We compare the AUC result of different methods under six privacy budgets: 1, 2, 3, 4, 5, and 6. The AUC results of all methods are illustrated in Fig.\,\ref{fig:Exp_PrivBud_on_lp_fig}. From this figure, we can see that AdvSGM are all significantly better than other private models in all privacy budgets across all datasets. This main reason is that our designed adversarial training module enhances the utility of skip-gram without compromising privacy. For DPAR, despite achieving higher AUC scores than DPGGAN, DPGVAE, and GAP across all datasets and privacy budgets, it still falls short of AdvSGM in terms of AUC score.
DPGGAN and DPGVAE use the MA mechanism \cite{abadi2016deep} to address the issue of excessive privacy budget splitting during optimization, but they still produce poor results. This occurs because these methods often converge prematurely under MA, especially with a limited privacy budget, leading to reduced performance in both privacy and utility.
GAP employs the aggregation perturbation (AP) technique to ensure differential privacy in GNNs, yet it also yields poor results. The main drawback of AP is its incompatibility with standard GNN architectures due to the high privacy costs involved. Conventional GNN models frequently query aggregation functions with each parameter update, requiring the re-perturbation of all aggregate outputs in every training iteration to maintain differential privacy. This process results in a significant increase in privacy costs.

\subsection{Impact of Privacy Budget on Node Clustering}\label{exp:eps_on_nclu}
The MI results for node clustering are shown in Fig.\,\ref{fig:Exp_PrivBud_on_NC_fig}. Note that we only evaluate MI on the PPI, Wiki, and Blog datasets, due to the absence of labeled data in the other datasets. From this figure, a key observation is that our AdvSGM consistently achieves the highest MI accuracy among all private methods. 
In conclusion, the results from these experiments strongly support our claim. The adversarial training component in AdvSGM enables it to strike an effective balance between privacy preservation and task utility. As such, AdvSGM proves to be a highly effective and versatile privacy-preserving approach that can be readily applied to a variety of downstream graph-based tasks, offering significant benefits in scenarios where both privacy and performance are critical considerations.

%% file: relatedWork.tex
\section{Related Work}\label{sec:Related_work}
% Related work of this paper includes differentially private deep learning and differentially private graph learning.
\textbf{Private Deep Learning.}
DP and its variant, local differential privacy, provide robust and strong privacy guarantees, even against adversaries with prior knowledge. These techniques have been extensively applied across diverse domains~\cite{cao2017quantifying, ye2019privkv, ye2023stateful, ye2021beyond, Zhang2023trajectory, ye2021privkvm}, with a notable surge in their adoption for privacy-preserving deep learning.
Early work by Dwork \emph{et al.}~\cite{dwork2010boosting} introduced the concept of strong composition, which aims to provide privacy guarantees when combining multiple queries or updates. However, as pointed out by Abadi \emph{et al.}~\cite{abadi2016deep}, existing composition theorems lack the necessary precision to accurately evaluate the privacy cost in deep learning models. To address this issue, the MA mechanism was introduced, which tracks the logarithmic moments of privacy loss variables and provides more accurate privacy loss estimates when combining Gaussian mechanisms under random sampling.
Mironov \emph{et al.}~\cite{mironov2019r} propose a new analysis method named RDP, which surpasses the MA mechanism and greatly enhances DPSGD performance. Further research~\cite{chen2020stochastic, nasr2020improving, papernot2021tempered, tramer2020differentially} explores modifications to model structures or learning algorithms, such as replacing traditional activation functions in CNNs with smoother \emph{Sigmoid} functions. Other studies~\cite{xiang2019differentially, yu2019differentially} focus on optimizations like adaptive gradient clipping bounds or dynamic privacy budget partitioning.  

\textbf{Private Graph Learning.} 
The most closely related work is by Ahuja \emph{et al.}~\cite{ahuja2020differentially}, which combines SGM and DPSGD~\cite{abadi2016deep} for private learning from sparse location data. However, this method doesn't apply to graphs due to noise from complex node relationships. Peng \textit{et al.}~\cite{peng2021differentially} propose a decentralized framework for privacy-preserving learning of embeddings from multiple knowledge graphs. Han \textit{et al.}~\cite{han2022framework} develop differentially private knowledge graph embeddings. Pan \textit{et al.}~\cite{pan2022fedwalk} present a federated framework for unsupervised node embedding with DP and high communication efficiency. Despite these advancements, like Ahuja \emph{et al.}, these methods face utility issues due to their perturbation mechanisms. Yang \textit{et al.}~\cite{abadi2016deep} develop differentially private GAN and VAE models for graph synthesis and link prediction, but these approaches often converge prematurely with limited privacy budgets, reducing both privacy and utility.
Epasto \textit{et al}.~\cite{epasto2022differentially} present a differentially private graph learning algorithm that outputs an approximate personalized PageRank and have provably bounded sensitivity to input edges. 
Unfortunately, the proposed private PageRank only achieves the weak edge-level DP. 
Another area of research in differential private graph learning focuses on GNNs~\cite{olatunji2021releasing, daigavane2021node, zhang2024dpar, sajadmanesh2023gap, sajadmanesh2023progap, xiang2023preserving, ran2024differentially}. The aggregation perturbation technique is often used to ensure DP in GNNs. Unlike traditional DPSGD algorithms, many differentially private GNN methods perturb the aggregate information from the GNN neighborhood aggregation step. Yet, the aggregation perturbation faces compatibility issues with standard GNN architectures, requiring the re-perturbation of all aggregate outputs at each training iteration, which leads to high privacy costs.

%% file: conclusion.tex
\section{Conclusion and Future Work}\label{sec:conclu}
In this paper, we have presented a differentially private skip-gram for graphs via adversarial training, called AdvSGM. The main features lie in two aspects. First, we design a novel adversarial training module by introducing two optimizable noise terms in activation functions. Second, we achieve DP during optimization by fine-tuning the weighs between modules. Furthermore, we demonstrate that AdvSGM obeys node-level differential privacy. Extensive experiments on real-world graph datasets demonstrate that our solution outperforms state-of-the-art competitors. In our future work, we plan to extend our method to attribute graphs and matrix factorization-based network embeddings.